\documentclass{nature}

\usepackage{graphics}
\usepackage{dcolumn}
\usepackage{url}
\usepackage[normalem]{ulem}
\usepackage{amsmath}
\usepackage{amsthm}
\usepackage[usenames,dvipsnames,svgnames,table]{xcolor}
\usepackage{rotating}
\usepackage[inline]{enumitem}
\usepackage{booktabs}
\usepackage{multirow}
\usepackage{censor}
\usepackage{listings}
\usepackage{cite}
\usepackage{listings}
\usepackage{color}
\usepackage{amsfonts}

\newcommand{\note}[1]{\textcolor{black}{#1}}

\usepackage{xspace}

\title{Pushing the Limits of Machine Design: Automated CPU Design with AI}

\author{Shuyao Cheng$^{1,2}$, Pengwei Jin$^{1,2}$, Qi Guo$^{1}$, Zidong Du$^{1,3}$, Rui Zhang$^{1,3}$, Yunhao Tian$^{1,2}$, Xing Hu$^{1,2}$, Yongwei Zhao$^{1,3}$, Yifan Hao$^{1}$, Xiangtao Guan$^{1,4}$, Husheng Han$^{1,2}$, Zhengyue Zhao$^{1,2}$,  Ximing Liu$^{1,2}$, Ling Li$^{5}$, Xishan Zhang$^{1,3}$, Yuejie Chu$^{1}$, Weilong Mao$^{1}$, Tianshi Chen$^{3}$ \& Yunji Chen$^{1,2,*}$}

\begin{document}

\maketitle

\begin{affiliations}
 \item State Key Lab of Processors, Institute of Computing Technology, Chinese Academy of Sciences
 \item University of Chinese Academy of Sciences
 \item Cambricon Technologies Corporation Limited
 \item University of Science and Technology of China
 \item Institute of Software, Chinese Academy of Sciences \\
 \\
 * Correspondence should be addressed to Yunji Chen (cyj@ict.ac.cn).
\end{affiliations}

\begin{abstract}
Design activity---constructing an artifact description satisfying given goals and constraints---distinguishes humanity from other animals and traditional machines, and endowing machines with design abilities at the human level or beyond has been a long-term pursuit. Though machines have already demonstrated their abilities in designing new materials, proteins, and computer programs with advanced artificial intelligence (AI) techniques, the search space for designing such objects is relatively small, and thus, ``Can machines design like humans?" remains an open question. To explore the boundary of machine design, here we present a new AI approach to automatically design a central processing unit (CPU), the brain of a computer, and one of the world’s most intricate devices humanity have ever designed. This approach generates the circuit logic, which is represented by a graph structure called Binary Speculation Diagram (BSD), of the CPU design from only external input-output observations instead of formal program code. During the generation of BSD, Monte Carlo-based expansion and the distance of Boolean functions are used to guarantee accuracy and efficiency, respectively. By efficiently exploring a search space of unprecedented size $10^{10^{540}}$, which is the largest one of all machine-designed objects to our best knowledge, and thus pushing the limits of machine design, our approach generates an industrial-scale RISC-V CPU within only 5 hours. The taped-out CPU successfully runs the Linux operating system and performs comparably against the human-designed Intel 80486SX CPU. In addition to learning the world’s first CPU only from input-output observations, which may reform the semiconductor industry by significantly reducing the design cycle, our approach even autonomously discovers human knowledge of the von Neumann architecture.
\end{abstract}


\newpage

\newtheorem{theorem}{\bf Theorem}
\newtheorem{definition}{\bf Definition}

``The proper study of mankind is the science of design." --- Herbert A. Simon, 1969\cite{simon1969sciences}

Design is an important human activity that involves problem-solving, planning, learning, and creativity. The design process can be modeled as finding an acceptable solution satisfying given goals and constraints from a high-dimensional search space\cite{simon1969sciences,simon1995problem,bashir1999estimating}. Moving the design process to machines not only unambiguously improves industrial productivity\cite{Amarel91IJCAI}, but also drives progress towards artificial general intelligence (AGI)\cite{solomonoff1964formal}.

Machine design is made possible for creating new objects such as materials\cite{Liu2017MaterialsDA, Tabor2018AcceleratingTD, Saal2020MachineLI, Lee2022MethodologicalFF,raccuglia2016machine}, proteins\cite{russ2020evolution}, drugs\cite{Schneider2018AutomatingDD, ztrk2020ExploringCS, Vamathevan2019ApplicationsOM, Chen2018TheRO, Elton2019DeepLF, Stokes2020ADL, JimenezLuna2020DrugDW}, chip floorplans\cite{Mirhoseini21Nature}, and computer programs\cite{gulwani2017program, chaudhuri2021neurosymbolic, doi:10.1126/science.abq1158}, with advanced artificial intelligence (AI) techniques, however, the search space of design process that machines can address is still far away from human ability. To our best knowledge, the search spaces of machine-designed materials, proteins, drugs, floorplan, and computer programs are at most about $10^{48}$, $10^{125}$, $10^{300}$, $10^{2500}$, and $10^{2998}$, respectively, while the search space of the human-designed small-size MD5 program\cite{gulwani2017program} even contains more than $10^{5973}$ alternatives, let alone well-recognized most complicated human-designed objects such as flights and integrated circuits. Specifically, designing the Airbus A380 flight is estimated to explore $(4E6)! \approx 10^{2.8E7}$ alternatives assuming a two-dimensional space because it has more than $4$ million individual components\cite{Official_A380}. Moreover, designing an industrial-scale integrated circuit such as a CPU (Central Processing Unit), which is the brain of a computer consisting of millions of well-organized nanoscale transistors, with $1798$ input variables and $1826$ output variables, requires exploring $2^{2^{1798} \times 1826} \approx 10^{10^{540}}$ alternatives\footnote{The input/output variables mainly include input/output data, instruction, data registers, control/status registers (CSRs), and the program counter (PC), etc., and the size of search space is determined by the number of all possible Truth Tables.}. Due to huge gaps in terms of search spaces between machine- and human-designed objects, the question \emph{``Can machines design like humans?"}, similar to the famous question \emph{``Can machines think?"} raised by Alan Turing in the 1950s\cite{turing1950computing}, remains an open question.

To explore the boundary of machine design, we aim at endowing machines with the ability to autonomously design the CPU, whose search space is much larger than that of all existing machine-designed objects. Once machines can design such an industrial-scale CPU without human intervention, it not only advances the revolution of the semiconductor industry by significantly boosting the design efficiency but also pushes the limits of machine design to approximate human performance. More importantly, the ability to design a machine by itself, i.e., self-designing, could serve as a foundational step towards building self-evolving machines.

Actually, though modern commercial electronic design automation (EDA) tools such as logic synthesis\cite{rudell1989logic} or high-level synthesis\cite{mcfarland1990high} tools are available to accelerate the CPU design, all these tools require hand-crafted formal program code as the input. Concretely, a team of talented engineers must use formal programming languages (e.g., Verilog\cite{inc1996ieee}, Chisel\cite{6241660chisel}, or C/C++\cite{Coussy09DT,Gajski12HLS}) to implement the circuit logic of a CPU based on design specification, and then various EDA tools can be used to facilitate functional validation and performance/power optimization of the circuit logic. The above highly complex and non-trivial process typically iterates for months or years, where the key bottleneck is the manual implementation of the input circuit logic in terms of formal program code.


To automate the CPU design without human programming, we consider using partial input-output examples only as the inputs, because they are directly accessible from a large number of legacy test cases. Therefore, the problem of automated CPU design can be formulated as generating the circuit logic in the form of a Boolean function satisfying the input-output specification. Due to the unprecedented size of the search space as stated, it is quite challenging to generate the target Boolean function potentially consisting of millions of lines of symbolic formulas. In addition, the generated Boolean function is almost zero tolerance to inaccuracy\footnote{In the industrial practice of designing the Intel Pentium 4 microprocessor, one billion tests (each with $10,000$ instructions) are executed during validation, indicating that validation accuracy should be over $99.99999999999\%$\cite{Bentley01DAC}.}, otherwise, CPUs will be malfunctioned and cause a huge amount of loss. A recent case in 2017 is that the Intel Atom C2000 bug affected many famous vendors, among which Cisco prepared 125M dollars to replace related products\cite{Intel2018}. Because of the above challenges, existing AI techniques fail to automatically learn a CPU design only from input-output examples. They are only capable of generating correct circuit logic at most about 200 logic gates\cite{rai2021logic,brayton1984logic,chen2012learning,roy2021prefixrl}, which is far less than that of industrial-scale CPUs, e.g., even the Intel 80486 CPU designed in the 1990s has 1.2 million transistors, equivalent to about 300,000 logic gates\cite{collen2011computer,kumar07SSCC}.

In this article, we report a RISC-V CPU automatically designed by a new AI approach, which generates large-scale Boolean function with almost $100\%$ validation accuracy (e.g., $>99.99999999999\%$ as Intel\cite{Bentley01DAC}) from only external input-output examples rather than formal programs written by the human. This approach generates the Boolean function represented by a graph structure called Binary Speculation Diagram (BSD), with a theoretical accuracy lower bound by using the Monte Carlo-based expansion, and the distance of Boolean functions is used to tackle the intractability. By efficiently exploring a search space of unprecedented size $10^{10^{540}}$, this approach manages to automatically generate an industrial-scale RISC-V CPU design within only 5 hours. We tape out the CPU and successfully run the Linux operating system and standard benchmarks on it. The CPU performs comparably against the human-designed Intel 80486SX CPU, while the design cycle is significantly reduced by about 1000×. This is the world’s first CPU designed by AI, pushing the limit of machine design and offering clear evidence of human-like design ability. This may further reform the semiconductor industry by significantly reducing the design cycles. We also demonstrate that our approach can discover not only the general von Neumann architecture\cite{von1993first} but also fine-grain architecture optimization from scratch, which sheds some light towards the machine’s self-evolution.
\section*{Learning the circuit logic of a CPU}
We aim at learning the correct circuit logic of a CPU design directly from input-output (IO) examples, which eliminates the intensive manual work of talented experts on iterative programming and validating the circuit logic based on informal natural-language-based design specification such as ISA (Instruction Set Architecture) documents (see Fig.\ref{fig:overview}a and Fig.\ref{fig:overview}b). 

In fact, the underlying problem is: with only finite inputs and their expected outputs (i.e., IO examples) of a CPU, inferring the circuit logic in the form of a large-scale Boolean function that can be generalized to infinite IO examples with high accuracy. The Boolean function is the intuitive expression of the circuit logic (see Fig.1c), and it can represent both combinational logic and sequential logic, which consists of combinational logic with registers. The controller of registers can also be represented by combinational logic. Therefore, with extra inner states of the registers, the entire CPU design consisting of combinational logic and sequential logic can be represented by Boolean functions.

\begin{figure}[t]
  \centering
  \includegraphics[width=.6\columnwidth]{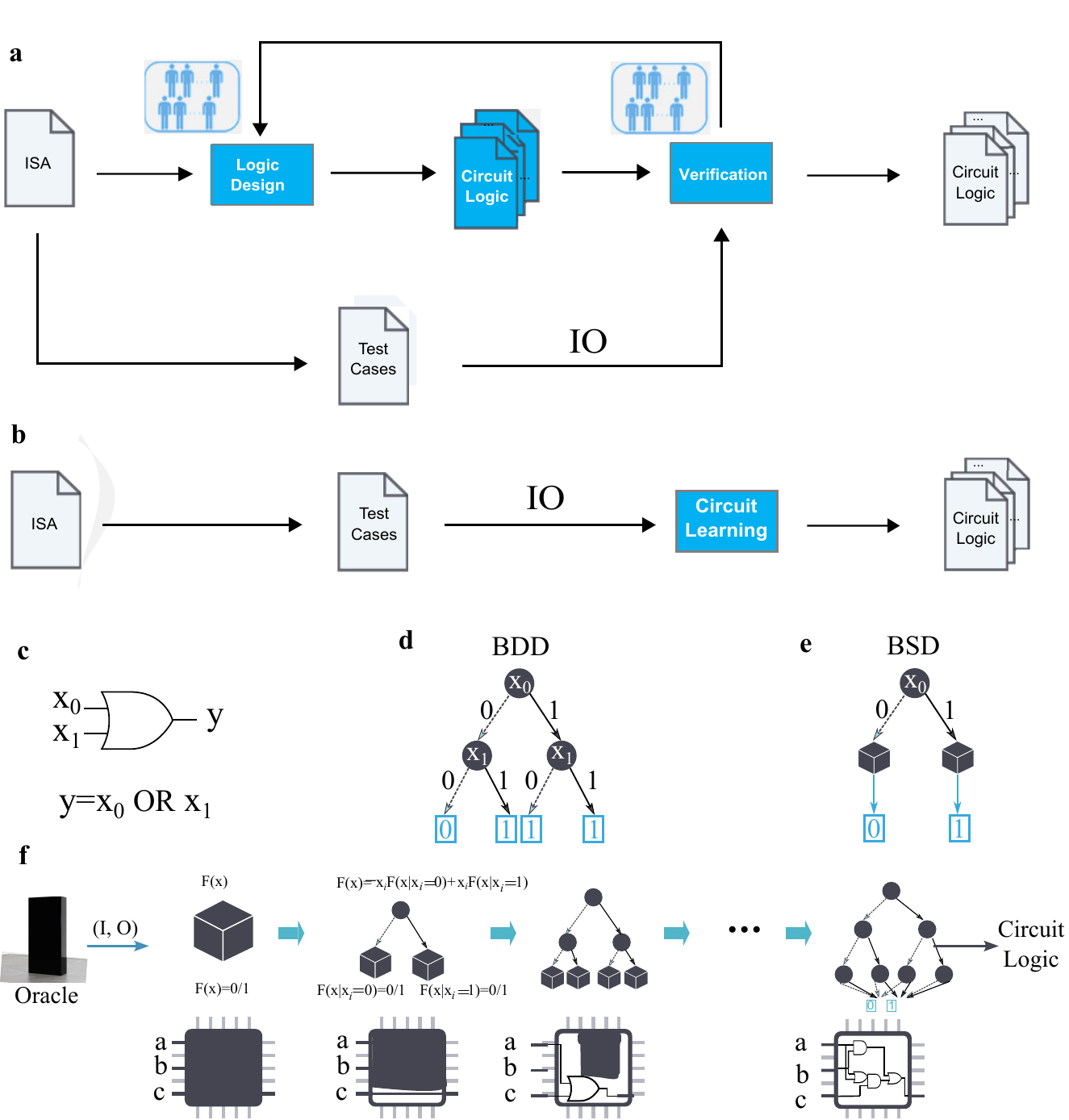}
  \caption{\textbf{CPU learning flow, circuit logic representation, and detailed learning process.} \textbf{a}, The traditional CPU design flow. \textbf{b}, The proposed flow to learn the CPU design from informal input-output examples. \textbf{c}, An example of the target circuit logic (i.e., OR gate) and the equivalent Boolean function. \textbf{d}, The graph representation of the Boolean function in the form of a Binary Decision Diagram (BDD). \textbf{e}, The graph representation of the Boolean function in the form of the proposed Binary Speculation Diagram (BSD). \textbf{f}, The detailed learning process. With only IO examples, a BSD of the target Boolean function is gradually generated from the root node to the leaf nodes, where the Monte Carlo-based policy is used to guide the generation process.}
  \label{fig:overview}
\end{figure}


One of the most well-known and efficient data structures for representing large-scale Boolean functions is Binary Decision Diagram (BDD), a rooted, directed acyclic graph (DAG) which consists of internal decision nodes and leaf nodes\cite{drechsler2013binary}. The internal decision node indicates a Boolean variable with the assignment of value 0 or 1 to its two child nodes, and the leaf nodes are labeled 0 and 1 (see Fig.\ref{fig:overview}d). The key advantage of BDD over other representations such as truth table, Karnaugh maps\cite{Brown1990Boolean}, or canonical sum-of-products form\cite{geetha2015network} is that BDD can reduce the description complexity of general Boolean functions from exponential to polynomial\cite{drechsler2013binary}. As a result, we reformulate this problem to \emph{generate an equivalent BDD representation of the target Boolean function}. 

This problem completely differs from conventional BDD generation problems tackled in existing EDA tools. The conventional BDD generation problem, such as logic synthesis, takes the formal circuit design, typically in terms of Verilog, as the input and generates the design implementation in terms of logic gates. Since existing BDD generation problems require the input Boolean function, they fail to address our problem. To address this problem, we propose to speculatively generate the final BDD with a new data structure called Binary Speculation Diagram (BSD) (see Fig. 1e), which represents a Boolean function by replacing the entire sub-trees (nodes and their child nodes) of conventional BDDs directly with a speculated constant 0 or 1. Such a feature of BSD enables fast and accurate simulation of large-scale Boolean functions with a much smaller size than conventional BDD. Nevertheless, the output of our problem can be treated as the input of traditional BDD generation problems such as logic synthesis, which can be further optimized by existing EDA tools such as the Synopsys Design Compiler.

With the help of the proposed BSD, this problem can be solved with a generating-by-searching regime; that is, a BSD to approximate the target Boolean function is gradually generated by searching the complicated relation of corresponding Boolean variables. The detailed flow of our approach is illustrated in Fig.\ref{fig:overview}f by using a $1$-bit circuit as an example. Initially, the generated BSD is empty, and hence it simulates the target circuit with speculated output value of constant $1$ or $0$, i.e., $\mathcal{F}(\mathtt{x})=1~or~0$, where $\mathcal{F}()$ is the corresponding Boolean function and $\mathtt{x}$ is the input. Then one bit is picked from the input $\mathtt{x}$, e.g., $x_i$, as the root node, and then it is expanded into two leaf nodes. The expansion follows the Boolean Expansion Theorem\cite{boole1854}: left node with $x_i=0$ and right with $x_i=1$. Hence the corresponding Boolean function of the generated BSD can be re-written as $\mathcal{F}(\mathtt{x})=\overline{x_i}\mathcal{F}(\mathtt{x}|x_i=0)+x_i\mathcal{F}(\mathtt{x}|x_i=1)$. The outputs of the two leaf nodes, $\mathcal{F}(\mathtt{x}|x_i=0)$ and $\mathcal{F}(\mathtt{x}|x_i=1)$, are also set to constant $0$ or $1$ speculated through the Monte Carlo-based policy. Following the basic idea of the Monte-Carlo Simulation, we sample n input-output pairs for every node. If all outputs of such IO pairs are 0, or all of such outputs are 1, the node is speculated as the constant function 0/1. The expansion on leaf nodes is repeated iteratively until the final BSD, which is a BDD as well, is generated. Note that along with the expansion, the generated BSD gets deeper, the unknown part of the target circuit gets smaller, the corresponding circuit logic gets more complex, and $\mathcal{F}(\mathtt{x})$ simulates the target circuit more accurately (proved in Methods). 

\section*{Addressing accuracy and scalability challenges}

During the generation of the BSD of the target Boolean function, our approach addresses two key challenges: (1) Accuracy: the inferred Boolean function must ensure extremely high accuracy for both given finite IO examples (i.e., $100\%$) and unseen infinite IO examples (e.g., $>99.99999999999\%$) and (2) Scalability: the inferred Boolean function might consist of millions of lines of symbolic formulas, so as to represent large-scale circuits such as CPU. These challenges are fully addressed by using the Monte Carlo-based expansion and the Boolean-distance-based node reduction, respectively. Because of these challenges, the proposed process cannot be addressed by the current EDA tools.

The first accuracy challenge, which must guarantee high accuracy for both given finite IO examples and unseen infinite IO examples, is addressed by using the Monte Carlo-based policy during the expansion of BSD. Concretely, whether a leaf node of the BSD will be expanded is based on two conditions: (1) 100\% correct on the already given IO examples, and (2) almost 100\% correct (i.e., $1-\epsilon$) on newly sampled IO examples\footnote{The traditional formal/semi-formal methods, which can guarantee 100\% accuracy, cannot be applied to our problem, because they require formal specification instead of input-output examples as the input. Nevertheless, our method can work well with existing formal/semi-formal EDA tools such as Formality.}.
The first condition guarantees that the generated BSD works definitely correctly on the given IO examples. The second condition guarantees a high accuracy when generalizing to unknown IO examples. The reason is that the IO examples randomly sampled by the Monte Carlo-based policy naturally capture the property of the target circuit. In other words, when the generated BSD can perform well on the randomly sampled IO examples, it should be able to satisfy the unknown IO examples. Besides, the leaf nodes of the current layer (e.g., $k$-th layer) can always be expanded to that of the $(k+1)$-th layer for higher accuracy. It can be proved that the accuracy of BSD expanded to the $(k+1)$-th layer will be no less than that of BSD expanded to the $k$-th layer, i.e., $Acc(\mathcal{F}_{k+1}) \ge Acc(\mathcal{F}_{k})$, where $\mathcal{F}_k$ is the Boolean function of the BSD when expanded to the $k$-th layer (see Theorem \ref{thm:decomacc} in Methods). Thus, the accuracy is boosted gradually during the process of hierarchical expansion. By increasing the input-output examples, including the random samples and samples with test errors, the accuracy of the automatically designed circuit can keep increasing.

The second scalability challenge, which requires discovering Boolean functions with millions of lines of symbolic formulas, is addressed by using the Boolean-distance-based node reduction. For the target CPU, the number of BSD nodes is reduced from $\sim10^{10^{540}}$ to $\sim10^6$. Roughly, the Boolean distance is used to measure the similarity between nodes (see Definition 2 in Methods), and those similar nodes are clustered together for potential node merging so as to significantly reduce the number of nodes. Note that the scalability challenge is further exacerbated by satisfying the stringent accuracy constraints since the high accuracy of the BSD is achieved by iteratively expanding its leaf nodes. 

\begin{figure}[t]
   \centering
   \includegraphics[width=1.\columnwidth]{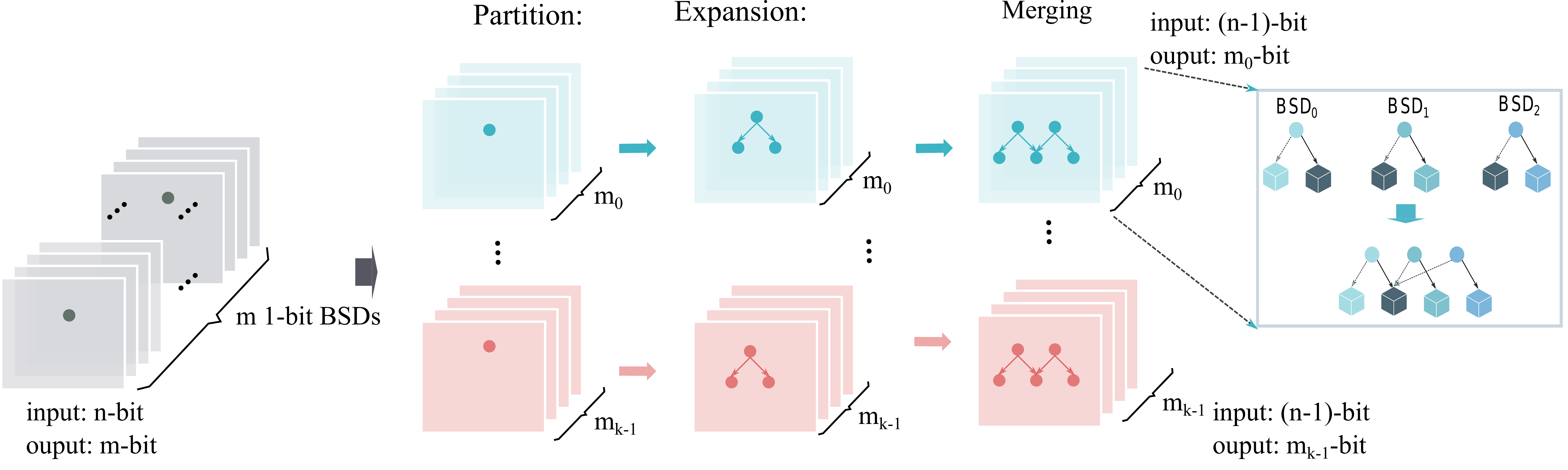}
   \caption{\textbf{One iteration of the proposed node-reduction policy.} The node reduction includes three stages, i.e., partition, expansion, and merging. In the partition stage, $m$ $1$-bit BSDs are grouped into $k$ clusters based on calculated Boolean distance. In the expansion stage, BSDs within a cluster are expanded with the same input bit to improve efficiency. In the merging stage, leaf nodes of BSD within a cluster, which perform exactly the same on IO examples, are merged.}
   \label{fig:nc}
\end{figure}

The proposed Boolean-distance-based node reduction consists of three key stages, including partition, expansion, and merging (see Fig.\ref{fig:nc}). The first partition stage groups similar BSDs ($m$-bit output circuit corresponds to $m$ $1$-bit BSDs) into the same cluster based on calculated Boolean distance. The BSDs grouped into one cluster share the same expansion order (i.e., expanding the leaf node with which input bit) to improve efficiency. Moreover, leaf nodes of BSDs within the same cluster have a high possibility for future node merging. The second expansion stage determines the expansion orders by measuring the Hamming distance---a natural similarity measure---between BSDs expanding on different input bits and randomly sampled IO examples. The expansion order is an important attribute of BDD, which determines the concrete BDD representation of the logic function. By comparing such distances, the input bit, which reduces the error the most, is selected as the root node for expansion. The third merging stage checks the leaf nodes of BSDs within the same cluster and then merges leaf nodes that perform exactly the same on the given IO examples and newly sampled IO examples. With two functionally equal nodes merging together once, all their son nodes are all considered functionally equivalent and only need to construct the same function once. Therefore, the total number of nodes in the BSD is significantly reduced with finite merging operations. Though the node similarity calculated based on Monte Carlo sampling is not exactly accurate and may bring errors after merging, we can guarantee that the merging error converges to zero when increasing the number of samples (see Theorem 2 in Methods).

\section*{Evaluation}
We demonstrate the capability of our approach by automatically generating a CPU design. The CPU with RISC-V 32IA instruction set is generated from a relatively small set of IO examples in less than $5$ hours, which can run the Linux operating system successfully and perform comparably against a human-designed CPU.

Concretely, the CPU has $1789$ input bits and 1826 output bits, and thus the total number of IO examples is $1826 \times 2^{1798}$, while only less than $2^{40}$ IO examples are randomly sampled for training. Also, the input-outputs of legacy programs are high-quality, while small test cases are used as training examples. The training process takes less than $5$ hours to achieve the accuracy of $>99.99999999999\%$ for validation tests. The generated CPU design then undergoes the physical design process with scripts at $65$nm technology to generate the layout for fabrication. Fig.\ref{fig:layout} illustrates the layout of the entire chip with major components marked, the manufactured chip with a frequency of $300$ MHz, and the printed circuit board containing the chip. Although we demonstrate the capability of our method with the RISC-V32IA instruction set, it is able to generate logic circuits of other CPU designs with different instruction sets as long as we can obtain the input-output examples as the input.

\begin{figure}[t]
  \centering
  \includegraphics[width=0.6\columnwidth]{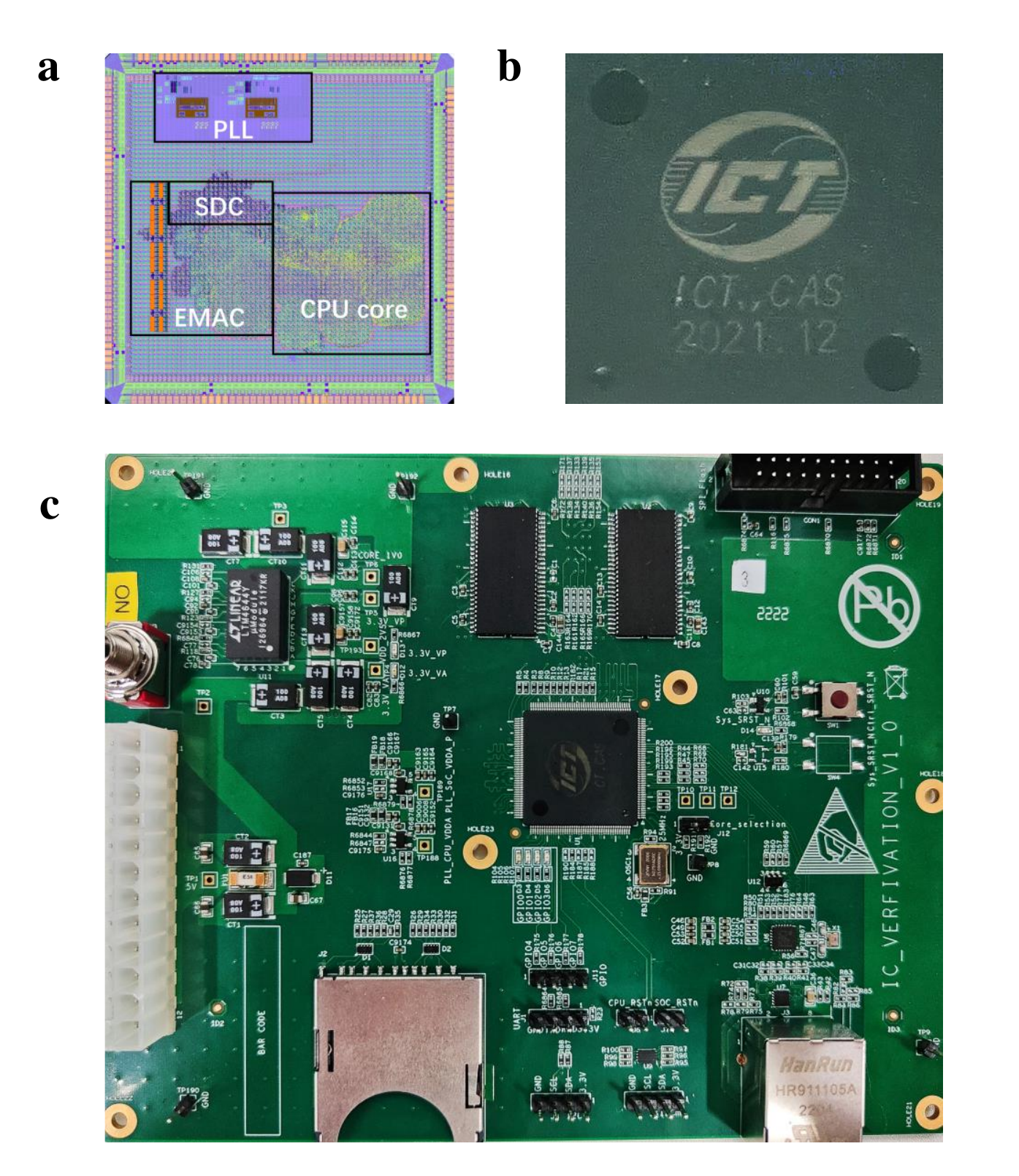}
  \caption{\textbf{The layout, manufactured chip, and printed circuit board.} \textbf{a}, the layout of the entire chip includes the core part of the CPU (CPU core), Phase Locked Loop for clock generation (PLL), Ethernet Media Access Controller module (EMAC), and SD controller (SDC). \textbf{b}, The manufactured chip after packaging can operate at a frequency of $300$ MHz. \textbf{c}, The printed circuit board containing the chip for functionality and performance evaluation.}
  \label{fig:layout}
\end{figure}

We successfully run the Linux (kernel 5.15) operating system and SPEC CINT 2000\cite{henning2000spec} on the generated CPU (i.e., CPU-AI) to validate the functionality (see Fig.\ref{fig:func_perf}a). We also use the widely-used Dhrystone\cite{weicker1984dhrystone} to evaluate the performance of CPU-AI. Fig.\ref{fig:func_perf}b compares the performance of CPU-AI against different generations of commercial CPUs, e.g., Intel 80386 (1980s), Intel 80486SX (1990s), and Intel Pentium III (2000s). On the evaluated program, it performs comparably to Intel 80486SX, designed in mid-1991. Though CPU-AI performs worse than recent processors such as Intel Core i7 3930K, it is the world’s first automatically designed CPU, and its performance could be significantly improved with augmented algorithms, which will be elaborated on later. 

\begin{figure}[t]
  \centering
  \includegraphics[width=\columnwidth]{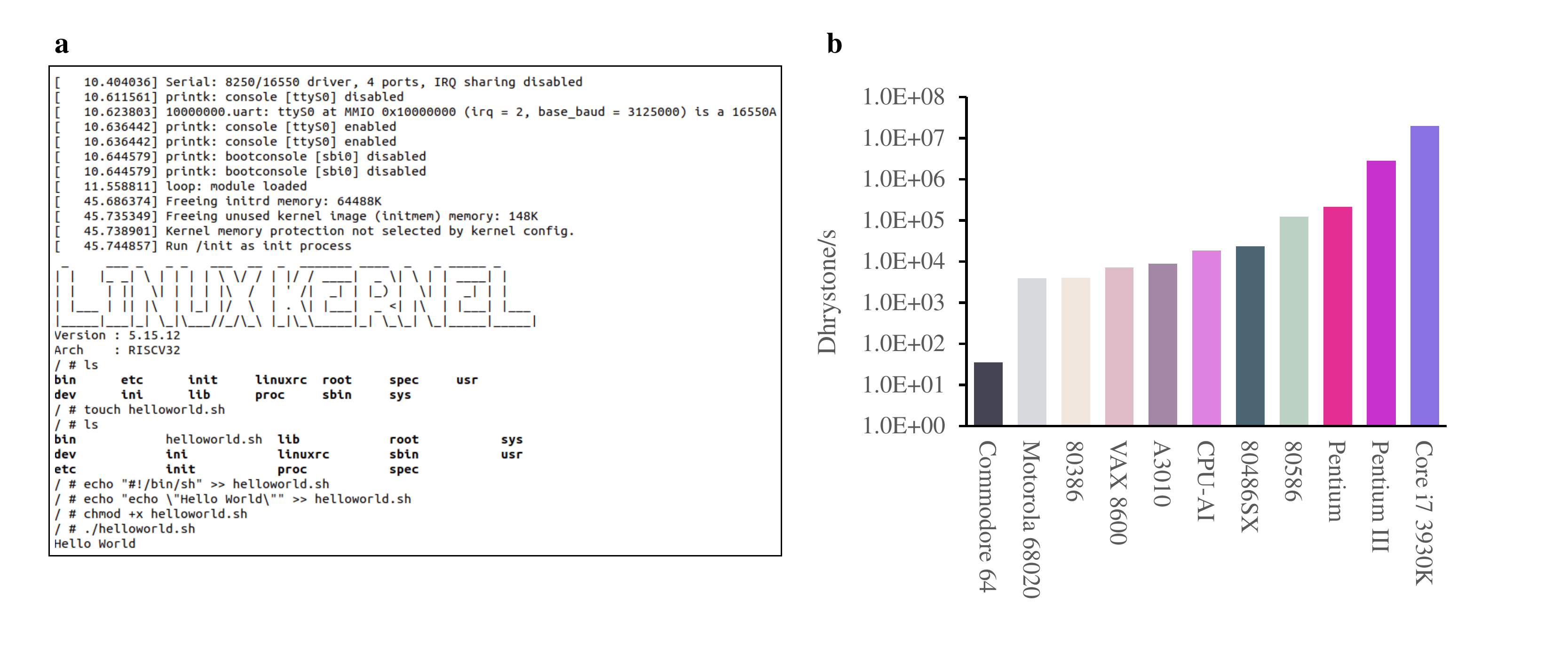}
  \caption{\textbf{Functional validation and performance comparison.} \textbf{a}, The outputs of booting up the Linux operating system. \textbf{b}, The performance of generated CPU (i.e., CPU-AI) is compared against commercial CPUs on the Dhrystone benchmark, and CPU-AI performs comparably to the human-designed Intel 80486SX CPU.}
  \label{fig:func_perf}
\end{figure}

To demonstrate that the proposed method reforms the CPU design flow and even potentially the semiconductor industry, we further compare the development costs of CPU-AI and human-designed CPU (i.e., CPU-Man), where the internal registers of CPU-Man exhibit exactly the same behavior as the CPU-AI based on the same instruction set specification. The CPU-Man takes $\sim5000$ man-hours to complete the entire design, while the CPU-AI only takes less than $5$ hours to obtain the design by training from the IO examples. The reduction of development costs is further validated by the design and verification costs of an Intel 486-compatible CPU, K486, which takes more than $190$ days (i.e., $4560$ hours) merely for the verification process\cite{yim1997design}. The underlying reason is that manual programming and verification of circuit logic, which consume more than $60\%$-$80\%$ of the design time and resources\cite{Bergeron12WT} in the traditional CPU design flow, is completely eliminated.

\section*{Discovering the von Neumann architecture}
By detailing the generated circuit logic of the taped-out CPU, we demonstrate that our approach discovers human knowledge of von Neumann architecture only from the input-output examples. Concretely, the generated CPU in terms of BSD is the key component of the von Neumann architecture, which mainly consists of the control unit, which is generated first in the BSD for global control, and the arithmetic unit (see Fig.\ref{fig:knowledge}a). The control unit generates controlling signals for the entire CPU, and the arithmetic unit accomplishes arithmetic operations (e.g., ADD and SUB) and logic operations (e.g., AND and OR). Moreover, we observe that both the control unit and arithmetic unit can be recursively decomposed into smaller functional modules such as instruction decoder, ALU, and LSU (load/store unit) by expanding more BSD layers (see Fig.\ref{fig:knowledge}b).

In addition to the above coarse-grain architecture, fine-grained architecture optimization can also be generated by using our approach. We use the well-known pipelining as an illustrative example. We provide more input and output variables, which are derived from the internal state of a pipelined CPU design, for training the proposed algorithm. Experimental results show that the algorithm is able to generate a CPU running at a higher frequency (i.e., 600MHz) by using the states of a 2-stage pipeline, along with the original input-output examples. The next step is to augment the algorithm to automatically generate such architectural modules without prior knowledge, i.e., additional internal states.

The proposed approach also has the potential to generate more aggressive (even unknown) fine-grained architectural optimization. An initial attempt is to apply fine-grained pipelining to a small number, even a single gate instead of the entire functional module, which is difficult for human designers while avoiding potential false data dependency for better throughput. Thus, by augmenting the proposed approach, we might endow the CPU with the ability to constantly improve its performance, which sheds some light towards the machine's self-evolution.

\begin{figure}[t]
  \centering
  \label{fig:knowledge}
  \includegraphics[width=1.0\columnwidth]{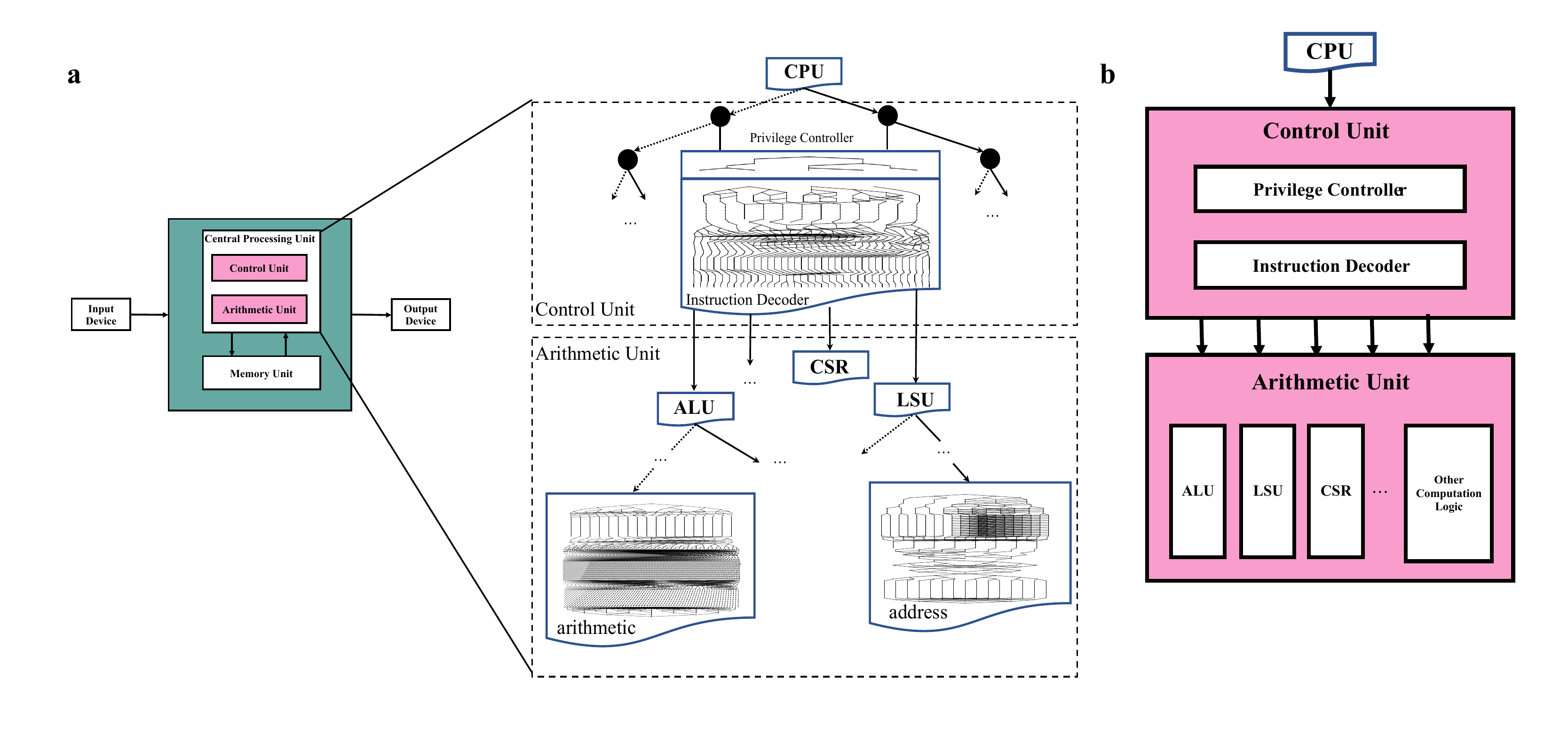}
  \caption{
    \textbf{Discovering the von Neumann architecture from scratch.} 
    \textbf{a}, The general von Neumann architecture can be discovered from the generated BSD, which mainly consists of the control unit and arithmetic unit. \textbf{b}, The control unit and arithmetic unit can be further decomposed into sub-modules in the BSD, e.g., the control unit contains the privilege controller and instruction decoder, the arithmetic unit contains ALU, LSU, and CSR (control and status registers), etc.}
  \label{fig:knowledge}
\end{figure}

\section*{Conclusion}
In this work, we propose a novel AI approach to obtain the world’s first automatically generated CPU, pushing the limits of machine design by exploring a search space of unprecedented size $10^{10^{540}}$, which is the largest one of all machine-design objects to our best knowledge, including materials, proteins, drugs, chip floorplans, and computer programs. The proposed approach generates the circuit logic of the CPU design in the form of a large-scale Binary Speculation Diagram (BSD) with $>99.99999999999\%$ accuracy from only external input-output observation. During the generation of BSD, the Monte Carlo-based expansion and the distance of Boolean functions are used to guarantee accuracy and efficiency, respectively. The generated CPU successfully runs the Linux operating system and performs comparably against the human-designed Intel 80486SX CPU. Moreover, compared to human-designed CPUs, our approach reduces the design cycle by about $1000\times$ because the manual programming and verification process of traditional CPU design is completely eliminated. Thus, our approach changes the traditional CPU design flow and potentially reforms the semiconductor industry. 

\note{In addition to offering human-like design abilities, our approach discovers human knowledge of von Neumann architecture. Our approach also has the potential to generate aggressive (even unknown) architecture optimization, which sheds some light on building a self-evolving machine to beat the latest CPU designed by humanity eventually. }



\newpage
\begin{methods}
We introduce the detailed methodology of the automated CPU design. 

\section*{Problem definition}
The underlying problem of automated CPU design is: with only \emph{finite} IO examples, 
infer the digital circuits in the form of \emph{Boolean function} that can be generalized to infinite IO examples with high accuracy. The intuition to use finite IO examples is that it is impossible to obtain all IO examples for large-scale circuits as their number increases exponentially ($2^n$ for $n$-bit input circuits). Moreover, the inferred Boolean function can represent both the combinational logic and sequential logic, as the sequential logic can be simply implemented as a combinational logic with registers. This problem can be generally formulated as follows:

\begin{definition}[Automated Chip Design]\label{def:ucl}
There is an \emph{Oracle} $\phi : \left\{0, 1\right\}^n \mapsto \left\{0, 1\right\}^m$ which has finite circuit complexity (defined by the circuit size) $C_\Omega(\phi) = c$. Given at most $N$ probes to the oracle $\left\{\left(\mathbf{x}_1, \phi(\mathbf{x}_1)\right), \left(\mathbf{x}_2, \phi(\mathbf{x}_2)\right),\right.$ $\dots,$ $\left.\left(\mathbf{x}_N, \phi(\mathbf{x}_N)\right)\right\}$, construct a Boolean function $\psi$ to simulate $\phi$, such that for a fraction of $1-\epsilon$ $(\epsilon \rightarrow 0)$ samples $\mathbf{x}$ in the input space $\left\{0, 1\right\}^n$, $\phi (\mathbf{x}) = \psi (\mathbf{x})$. 
\end{definition}


The problem of automated chip design exhibits the following two challenges: (1) 
Accuracy: The inferred Boolean function must ensure the correctness for the given IO examples with certainty and high accuracy for unseen IO examples (e.g. $99.99999999999\%$), and (2) 
Scalability: The inferred Boolean function has a comparable scale with an industrial CPU, which might consist of millions of lines of formulas. 

\section*{Overview of our approach}

As stated, achieving automated chip design confronts both accuracy and scalability challenges. To solve this problem, we propose the Binary Speculation Diagram (BSD), an equivalent form of a Boolean function but with a hierarchical structure for searchability, to represent the circuit logic of CPU design. Then, we reformulate this problem to generate the corresponding BSD with negligible accuracy loss and incurring a small computing overhead. Specifically, the BSD structure is shown in Fig.\ref{fig:overview}c. BSD is a rooted directed acyclic graph with non-leaf (decision) nodes and leaf nodes. Each non-leaf node represents a selected input bit, and each edge represents a specific assignment of $0$ or $1$ for its non-leaf source node; each leaf node represents a speculated $1$-bit constant Boolean function, and each node with all its sub-nodes represents a $1$-bit Boolean function. With the expansion of the leaf nodes, the generated BSD gets deeper, the unknown part of the target circuit gets smaller, the corresponding circuit logic gets more complex, and the generated BSD simulates the target circuit more accurately. The BSD node sizes indicate how large the circuit area is overhead, and it defines the circuit complexity.  

However, generating a CPU-scale BSD remains significantly challenging due to the searching space explosion. Hence, we propose the Boolean-distance-based node reduction method to address the scalability challenge while maintaining extremely high accuracy, which includes three iterative stages: \emph{partition}, \emph{expansion}, and \emph{merging}, as shown in Fig.\ref{fig:nc}.  In the simplified example to build a circuit with $n$-bit inputs and $m$-bit outputs, the BSD is initialized with $m$ leaf nodes, i.e. $m$ $1$-bit constant Boolean functions of speculated output values, corresponding to $m$-bit outputs of the target circuit. The \emph{partition} stage groups the outputs of the target circuit into several clusters sharing the same expansion order for improving the generation efficiency of BSD. In the \emph{expansion} stage, during the $(k+1)$-th expansion process, each leaf node of the $k$-th layer is expanded into two new leaf nodes following the Boolean Expansion Theorem.
The expansion process based on the Monte Carlo-based policy guarantees that the accuracy of learned BSD on unseen IO examples definitely improves. 
The \emph{merging} stage reduces the scale of the generated BSD by merging leaf nodes with the exact same function for given IO examples and newly sampled IO samples. 
The detailed algorithm of these three stages is described as follows, along with the accuracy and performance analysis.

\section*{Detailed algorithm}
In this section, we present the detailed stages of each iteration for generating the BSD.

\subsection{Partition.}
The \emph{partition} stage groups similar BSD nodes into one cluster, aiming to reduce the computing overhead of BSD generation and increase the potentiality for merging leaf nodes.
Concretely, in each iteration of building the BSD, before expanding the nodes in the $k$-th layer, the partition stage first separates BSD nodes into several clusters, and BSD nodes in the same cluster will share the same expansion order in the following \emph{expansion} stage. BSD nodes in two different clusters are sufficiently different from the circuit view, and thus expanding them with the same expansion ordering will lead to a BSD structure with larger sizes.
The cluster grouping is based on a special distance, named Boolean distance, which indicates the similarity of two circuits. Boolean distance is calculated based on the circuit complexity, as explained  in Definition \ref{def:mcd}. 
Considering that the nodes in the $k$-th layer are currently approximated with the constant function whose values are decided by IO examples from Monte Carlo sampling, the circuit complexity of the nodes is also approximated with IO examples.

\begin{definition}[Boolean distance]\label{def:mcd}
Given two Boolean functions, $f$ and $g$, the Boolean distance $Dist(f,g)$ is calculated based on the circuit complexity of $f$, $g$, and their combination $\tau$.
\begin{equation}
\mathit{Dist}(f,g) = C_\Omega(f) + C_\Omega(g) - C_\Omega(\tau),
\end{equation}
where $C_\Omega(\cdot)$ is the circuit complexity of a Boolean function, $\tau$ is the combination of $f$ and $g$, the input bits of $\tau$ are the union of input bits of $f$ and $g$, the output of $\tau$ is the concatenation of the output bits of $f$ and $g$, formulated as
\begin{equation}
\tau(\mathbf{x}_{f\backslash g}, \mathbf{x}_{g \backslash f}, \mathbf{x}_{f\cap g}) = (f(\mathbf{x}_{f}), g(\mathbf{x}_{g})),
\end{equation}
where $\mathbf{x}_{f \backslash g}$ are the input bits that in $f$ but not in $g$, $\mathbf{x}_{g\backslash f}$ are the input bits that in $g$ but not in $f$, $\mathbf{x}_{f\cap g}$ are the input bits that in both $f$ and $g$, $(\cdot,\cdot)$ means the concatenation of the output bits of two Boolean functions.

\end{definition}

\begin{figure}[t]
   \centering
   \includegraphics[width=.75\columnwidth]{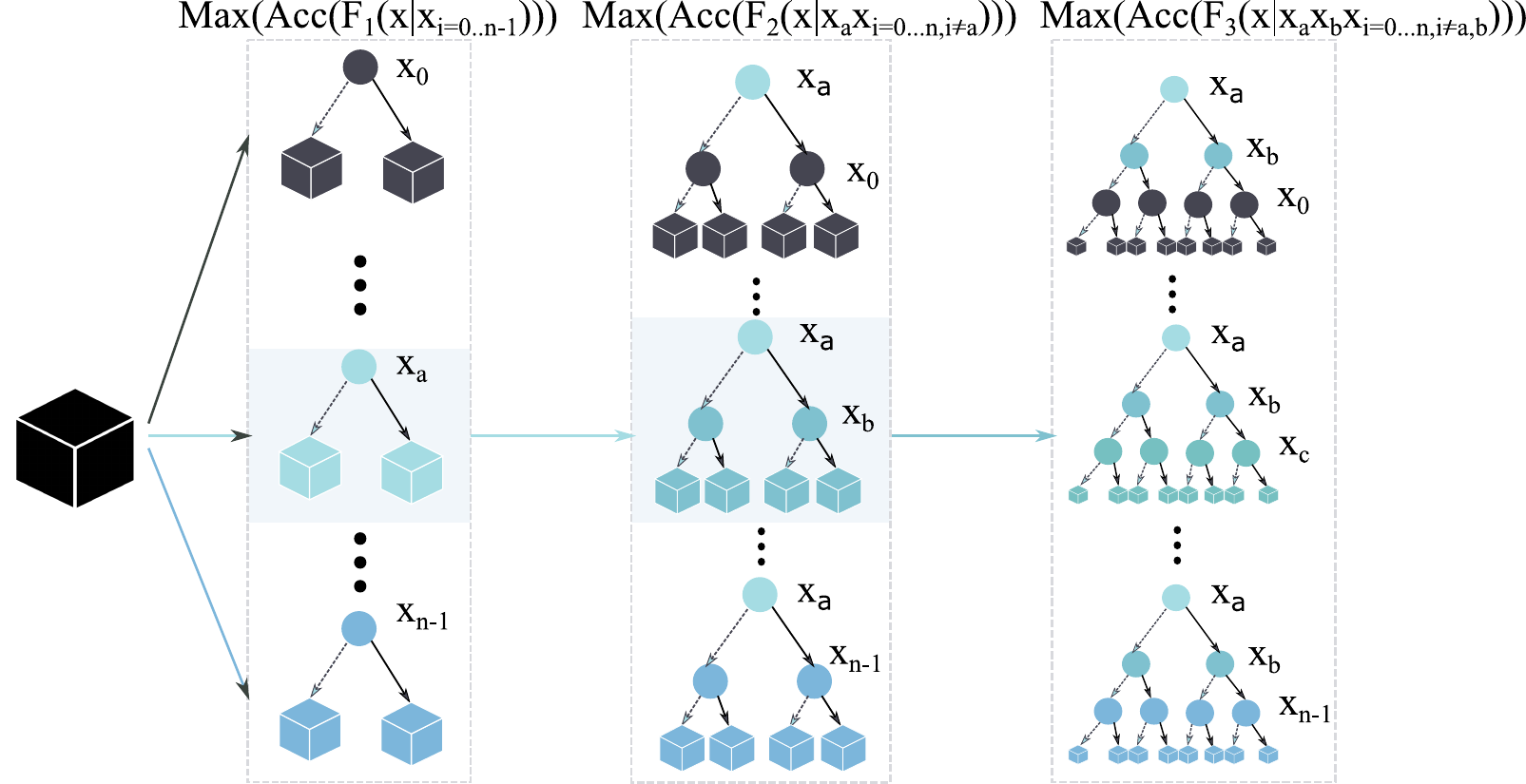}
   \caption{\textbf{An example of determining the expansion order within one BSD cluster.} In each iteration of BSD expansion, the input bit selected for expansion is decided by improving the accuracy the most. In the first iteration, $x_a$ is selected for expansion as the accuracy is improved the most by comparing the accuracy of expanding each input bit from $x_0$ to $x_{n-1}$. In the second iteration, the input bit $x_b$ is selected for expansion in the same manner.}
   \label{fig:mc2}
\end{figure}

\subsection{Expansion.} 
The \emph{expansion} stage is the key part for BSD generation with assured accuracy. The Boolean Expansion Theorem assures 100\% correctness on the given IO examples. For unseen samples, we use Monte Carlo-based policy to achieve almost 100\% correctness of BSD. 
In the $k$-th iteration, we approximate the accuracy of the BSD $Acc(\mathcal{F}_k)$ to the Hamming distance between $\mathcal{F}_k(\mathbf{x})$ and the target circuit based on sampled IO examples $\mathbf{x}$.
The accuracy $Acc(\mathcal{F}_{k+1})$ will be no less than $Acc(\mathcal{F}_k)$ after expansion using Theorem \ref{thm:decomacc}.
We select the input bit, which can mostly increase the accuracy of $Acc(\mathcal{F}_{k+1})$ for expansion (see Extended Data Fig.\ref{fig:mc2}).
Because the more the increase of accuracy after each expansion, the less the number of expansions, selecting the input bit, which mostly increases the accuracy, will reduce the number of expansions and improve the efficiency of building the BSD. In this way, we find the optimal expansion order using the approximated accuracy based on Hamming distance. When finishing the generating process of BSD, the expansion stage guarantees that the generated BSDs are definitely correct on the given IO examples. Since the IO examples are randomly sampled by Monte Carlo-based policy, these IO examples contain the property of the target circuit and guarantee a high accuracy when generalizing to unknown IO examples.

\begin{theorem}[The accuracy of BSD boosts after expansion]\label{thm:decomacc}
After expanding the generated BSD  $\mathcal{F}_k$ by any input bit $x_i$ to $\mathcal{F}_{k+1}$, the accuracy of expansion ended with $\mathcal{F}_k$ will be no larger than the accuracy of expansion ended with $\mathcal{F}_{k+1}$, that is, 
\begin{equation}\label{eq1}
    Acc(\mathcal{F}_k) \leq  Acc(\mathcal{F}_{k+1}).
\end{equation}
\end{theorem}
\begin{proof}
Consider a node $f_k$ in $\mathcal{F}_k$ when the expansion is ended with $\mathcal{F}_k$.
Denote $Q_0(f_k)$ and $Q_1(f_k)$ as the proportions that the value of $f_k$ are equal to 0 and 1, respectively. Thus, the larger one between $Q_0(f_k)$ and $Q_1(f_k)$ decides the value of $f_k$
\begin{equation}\label{eq2}
    Acc(f_k)  =  \max(Q_0(f_k), Q_1(f_k)).
\end{equation}
We continue expand $f_k$ to $f_{k+1}^{0}$ and $f_{k+1}^{1}$ with a selected input bit $x_i$ assigning to $0$ and $1$, respectively. Denote $P_0(f_k)$ and $P_1(f_k)$ as the proportions that $f_k$ are expanded with $x_i$ assigning to $0$ and $1$, respectively, then we have
\begin{equation}
    \begin{array}{lll}
    & &Acc(\{f_{k+1}^{0}, f_{k+1}^{1}\}) \\
    &=& P_0(f_k) Acc(f_{k+1}^{0}) + P_1(f_k) Acc(f_{k+1}^{1}) \\ 
    &=& P_0(f_k) \max(Q_0(f_{k+1}^{0}), Q_1(f_{k+1}^{0})) + P_1(f_k) \max(Q_0(f_{k+1}^{1}), Q_1(f_{k+1}^{1})).
    \end{array}
\end{equation}
We also have
\begin{equation}\label{eq4}
    Q_0(f_k) = P_0(f_k)Q_0(f_{k+1}^0) + P_1(f_k)Q_0(f_{k+1}^1),
\end{equation}
\begin{equation}\label{eq5}
    Q_1(f_k) = P_0(f_k)Q_1(f_{k+1}^0) + P_1(f_k)Q_1(f_{k+1}^1).
\end{equation}
Combine Eqn. \eqref{eq2} \eqref{eq4} \eqref{eq5}, we can get
\begin{equation}\label{eq6}
\begin{array}{lll}
Acc(f_k) & = & 
\max(P_0(f_k)Q_0(f_{k+1}^0) + P_1(f_k)Q_0(f_{k+1}^1), P_0(f_k)Q_1(f_{k+1}^0) + P_1(f_k)Q_1(f_{k+1}^1)) \\
& \leq & \max(P_0(f_k)Q_0(f_{k+1}^0), P_0(f_k)Q_1(f_{k+1}^0)) + \max(P_1(f_k)Q_0(f_{k+1}^1), P_1(f_k)Q_1(f_{k+1}^1)) \\
& = & P_0(f_k) \max(Q_0(f_{k+1}^{0}), Q_1(f_{k+1}^{0})) + P_1(f_k) \max(Q_0(f_{k+1}^{1}), Q_1(f_{k+1}^{1})) \\
& = & Acc(\{f_{k+1}^{0}, f_{k+1}^{1}\}).
\end{array}
\end{equation}
When we apply Eqn. \eqref{eq6} to all nodes in $\mathcal{F}_k$, we can directly obtain Eqn. \eqref{eq1}.
\end{proof}

\subsection{Merging.} 
The \emph{merging} stage eliminates the redundancy of isomorphic subgraphs in the generated BSD and reduces both the computing overhead and BSD sizes. After the stage of \emph{expansion}, the merging stage merges the similar leaf nodes, which are determined by the Boolean distance. Specifically, two leaf nodes $f_{(k+1),p}$ and $f_{(k+1),q}$ whose outputs are exactly equal, i.e. their Hamming distance is zero $f_{(k+1),p}(\mathbf{x})=f_{(k+1),q}(\mathbf{x})$, will be merged. Although the determination of merging with Monte Carlo sampling is not exactly accurate and may bring error after merging, using Theorem \ref{thm:reductionerror} we can guarantee that the merging error is controllable and will converge to zero when we increase the number of sampling.

\begin{theorem}[The convergence of merging error]\label{thm:reductionerror}
When we perform the merging stage during the process of generating BSD, the probability that the error of generated BSD larger than $\delta$ will no more than $\frac{T}{K\delta}$, where $T$ is the total number of merging stages, $K$ is the number of IO examples during calculating the Boolean distance, and $\delta$ is a very small value. Thus, the probability that the error of generated BSD is larger than a very small value will converge to zero by increasing the number of sampled IO examples.
\begin{proof}
Assuming in one merging stage, we merge two leaf nodes that correspond to two not exactly the same Boolean functions. We denote $r$ as the proportion of incorrect results of merged nodes after the merging stage, namely the proportion of different results of the two corresponding Boolean functions given the same input. Denote $R$ as the error rate of merged nodes after any merging stage, then the distribution of $R$ is
\begin{equation}
    Pr(R) = \left\{
    \begin{array}{ll}
        (1-r)^K & R = r \\
        1-(1-r)^K & R = 0 \\
    \end{array}
    \right..
\end{equation}
So the expected value of $R$ is
\begin{equation}
    E(R) = r(1-r)^K.
\end{equation}
Considering the range of $r$ is $(0, \frac{1}{2}]$, $E(R)$ increases monotonically at $r\in(0,\frac{1}{K+1})$, and decreases monotonically at $r\in(\frac{1}{K+1},\frac{1}{2}]$. Therefore, $E(R)$ takes the maximum value $\frac{1}{K+1}\left(\frac{K}{K+1}\right)^K$ when $r$ takes $\frac{1}{K+1}$, so
\begin{equation}
    E(R) \leq \frac{1}{K+1}\left(\frac{K}{K+1}\right)^K < \frac{1}{K}.
\end{equation}
Given a small value $\delta$, using Markov's inequality, we can obtain
\begin{equation}\label{eq11}
    Pr(R\geq\delta) \leq \frac{E(R)}{\delta} < \frac{1}{K\delta}.
\end{equation}
Eqn. \eqref{eq11} means that in each merging stage, the probability that the error rate after merging two leaf nodes larger than $\delta$ is less than $\frac{1}{K\delta}$. Thus, the probability that the accuracy of a merged node larger than $1-\delta$ is no less than $1-\frac{1}{K\delta}$. When finishing the process of generating the BSD, which contains $T$ merging stages, the probability that the accuracy of the final BSD larger than $1-\delta$ is no less than $(1-\frac{1}{K\delta})^T$ which is no less than $1-\frac{T}{K\delta}$. In conclusion, the probability that the error of generated BSD is larger than $\delta$ will be no more than $\frac{T}{K\delta}$.
\end{proof}
\end{theorem}

\section*{An illustrative example}
Here we use the 8-bit adder as an illustrative example to detail the process of the proposed BSD generation algorithm, including partition, expansion, and merging (see Fig.\ref{fig:illustrative example}A), where $c={c_8,..,c_0}$ are the target Boolean functions of corresponding output bits of adding two 8-bit inputs $\{a_7,..,a_0\}$ and $\{b_7,...,b_0\}$. 

\begin{itemize}[itemsep=0pt,topsep=0pt,leftmargin=5pt]
\item{\textbf{Partition.}} First, all the target Boolean functions computing output bits (black nodes in Fig. ~\ref{fig:illustrative example}) are partitioned into different groups, e.g., the two most significant bits (MSB), $c_8$ and $c_7$, are partitioned in one group. As stated, the partition is performed based on the Boolean distance between the corresponding output bits, and a larger distance means that there are more similarities and merging possibilities for the output bits. In the cluster consisting of $c_8$ and $c_7$, the distance between $c_8$ and $c_7$ is $Dist = C_\Omega(c_8) + C_\Omega(c_7) - C_\Omega(c_8, c_7) = 23+43-46 = 20$, which is the largest distance result among all of the possible 2-bit clusters. For example, if $c_8$ and $c_4$ are partitioned into one cluster, the corresponding distance is $Dist = C_\Omega(c_8) + C_\Omega(c_4) - C_\Omega(c_8, c_4) = 23+25-37 = 11 < 20$. With the help of $Dist$, the proposed method partitions the BSD nodes into clusters with more merging possibilities and thus improves the following merging efficiency to obtain compact Boolean functions.
\item{\textbf{Expansion.}} Second, each output bit is expanded into two sub-functions on one input bit. For example, the most significant bit, $c_8$, is expanded into two sub-functions, $c_{8|a_0=0}=\{a_7,..,\mathbf{0}\}+\{b_7,...,b_0\}$ and $c_{8|a_0=1}=\{a_7,..,\mathbf{1}\}+\{b_7,...,b_0\}$, and thus $c_8=c_{8|a_0=0}+c_{8|a_0=1}$. After this expansion, the 16-bit-input 1 bit-output function $c_i$ is presented with two 15-bit-input 1 bit-output sub-functions, i.e., $c_{i|a_0=0}$ and $c_{i|a_0=0}$, $i\in{0,..,8}$. These sub-functions will be further expanded on another input bit to generate four 14-bit-input 1-bit-output sub-functions. According to \textbf{Theorem 1}, the accuracy of the BSD keeps increasing with more expansion layers. To improve the accuracy more efficiently, it is important to decide which bit to expand. In our method, the bit expansion order is the same within each partitioned cluster. In the cluster consisting of $c_8$ and $c_7$, the first two bits to expand are $a_7$ and $b_7$, the MSB of the operands, because the speculation accuracy increases the most by expanding these bits. Concretely, we use Hamming distance to approximate the accuracy improvement. For every expanded BSD, we sample 400 IO examples to obtain an output vector with 400 Boolean values and then calculate the Hamming distance between this output vector with that of the function without expansion. For the first layer, the expansion variable is $a_7$, with the largest Hamming distance of 155. After the expansion of the first layer, there are no merging possibilities in the sub-functions, and then we expand the second layer. For the second layer, the expansion variable is $b_7$, with the largest Hamming distance of 155 as well.

\item{\textbf{Merging.}} Third, sub-functions in a cluster can be merged when their Hamming distance is zero, which means that the functions are equal. In this case, for every sub-function, we sample 10000 IO examples to obtain a Boolean vector with 10000 Boolean values. Then, sub-functions belong to $c4$ with Hamming distance of 0 are merged, where
$c_{4|a_4=0,b_4=0}$ and $c_{4|a_4=1,b_4=1}$ are merged (orange dotted lines), and $c_{4|a_4=0,b_4=1}$ and $c_{4|a_4=1,b_4=0}$ are merged (blue dotted lines). Moreover, sub-functions do not belong to the same output bit but the same cluster can also be merged if their Hamming distance is 0. Thus, sub-functions belong to $c_8$ and $c_7$ are merged, where
$c_{8|a_7=0,b_7=0}$ and $c_{8|a_7=1,b_7=1}$ are merged into constant function 0/1 (black dotted lines),$c_{8|a_7=1,b_7=0}$ ,$c_{8|a_7=0,b_7=1}$,$c_{7|a_7=1,b_7=0}$and $c_{7|a_7=0,b_7=1}$ are merged (orange dotted lines), and $c_{7|a_7=0,b_7=0}$ and $c_{7|a_7=1,b_7=1}$ are merged (blue dotted lines). Such merging significantly reduces the total nodes for the final function. In the partitioned cluster with $c_8$ and $c_7$, the 8 nodes at layer 2 are merged into 2 nodes. In the partitioned cluster with only one output bit $c_4$, the 4 leaf nodes at layer 2 are also merged into 2 nodes.
\end{itemize}
By repeating the above partition, expansion, and merging stages, the final compact BSD representation with 139 nodes can be obtained (see Fig.\ref{fig:illustrative example}B). As a comparison, without the proposed partition and merging stage, and the variable orders are randomly chosen during expansion, the final BSD will contain about $2^{16}=65536$ BSD nodes.

 \begin{figure}[t]
   \centering
   \includegraphics[width=0.5\columnwidth]{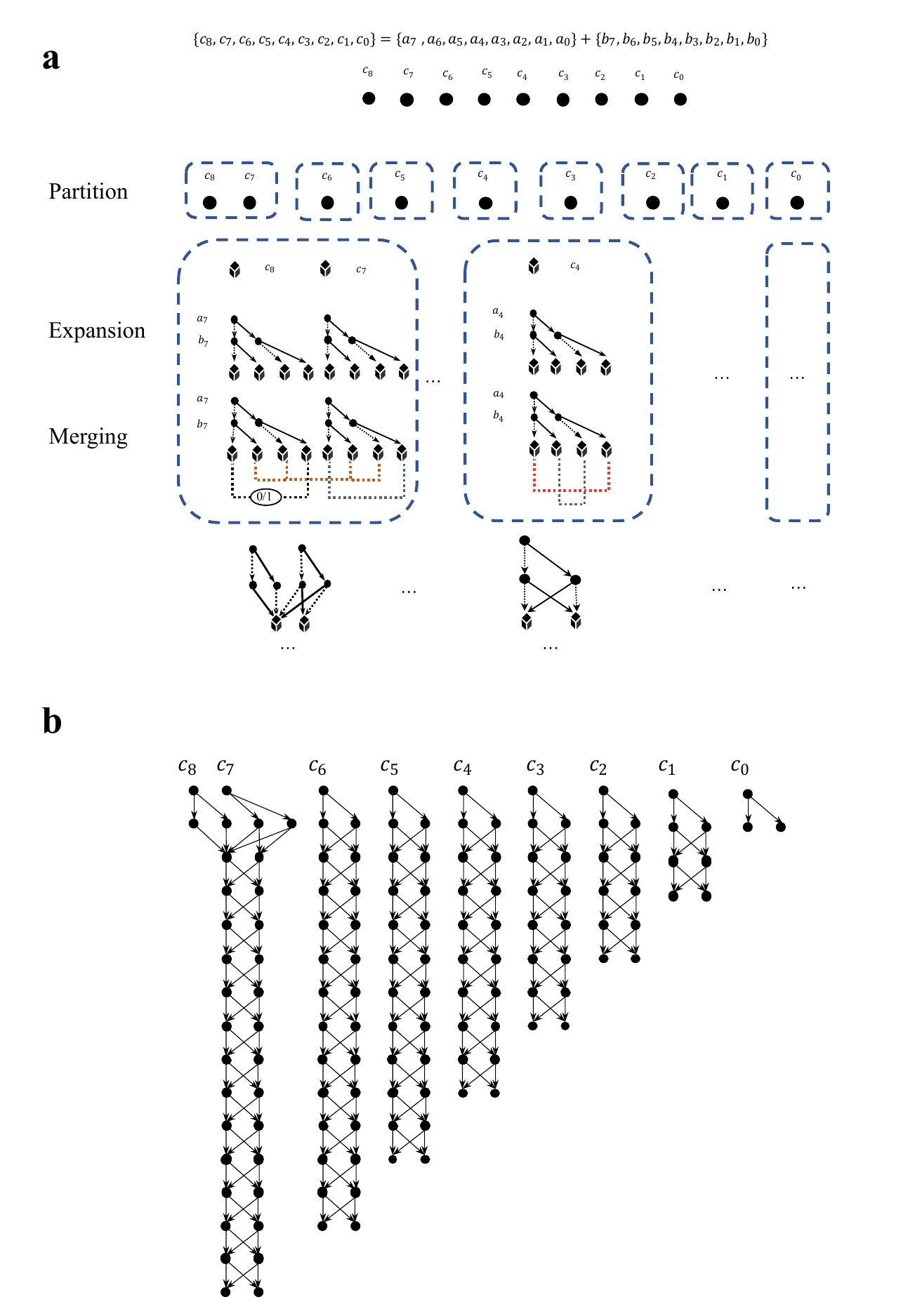}
   \caption{
   \textbf{An 8-bit full adder example of BSD generation.} \textbf{a}, BSD generation process of the 9-bit output with partition, expansion, and merging. \textbf{b}, The final BSD presentation of the 8-bit full adder. }
   \label{fig:illustrative example}
 \end{figure}

\section*{Evaluation}

In this section, we introduce detailed settings and the main results of automated CPU design.

\subsection{Detailed design configurations.}

\emph{Data for BSD generation:}
The IO examples for automated CPU design are directly borrowed from the verification of traditional design flow. There are two main kinds of IO examples for traditional verification. The first kind is the randomly generated input stimuli and their corresponding outputs, which can be easily obtained from a functional simulator\cite{bellard2005qemu} or a CPU with the identical ISA (instruction set architecture). The second kind is the input-outputs of legacy programs with particular functionalities, such as high-quality test cases. We start from RISC-V official ISA documents and end with a verification-passed GTECH netlist after synthesis. 

\emph{Algorithmic parameters}: During the generation of circuit logic in the form of BSD, there are multiple parameters to set. For the partition stage, the maximal number of clusters for each output bit is set as $10$ to control the scale of the generated BSD. For the expansion stage, the maximal width of BSD is set as $10,000$ to balance the BSD accuracy and expansion efficiency. For the merging stage, to determine the similarity between different nodes, the maximal sampling number per node is set as $1,000,000$ to balance the reduction accuracy and sampling efficiency. The implemented program is executed on a Linux cluster including 68 servers, each of which is equipped with 2 Intel Xeon Gold 6230 CPUs.

\emph{Hardware implementation:}

After the generation of BSD, we tape out the CPU chip to demonstrate the effectiveness of our approach. Specifically, the generated BSD is converted into a Verilog module by traversing every node in it. Hence, the BSD generation is compatible with the existing back-end design flow.
We use commercial tools to transform the Verilog module to a GTECH netlist, better supported by back-end EDA flows. We verify our output netlist on the FPGAs and tape out the chip with 65nm technology. The automatically designed CPU was sent to the manufacturer in December 2021.

\begin{table*}
  \centering
  \footnotesize
  \caption{Hardware characteristics of the generated CPU.}
  \label{tab:hwchar}
  \begin{tabular}{lllll}
    \toprule
    \textbf{Component} & \textbf{Area ($um^2$)} & \textbf{(\%)} & \textbf{Power (mW)} & \textbf{(\%)} \\
    \midrule

    CPU core & 275933.53&&14.46 &\\

    \midrule

    Combinational &264476.88&95.85&8.70&60.14\\

    Register &11456.64&4.15&5.76&39.86\\

    \bottomrule   
 \end{tabular}
\end{table*}

\subsection{Functionality and performance.}

To validate both the functionality and performance of the generated CPU before fabrication, we use multiple FPGAs to guarantee that the accuracy of $>99.99999999999\%$ is achieved for the validation tests, which consist of more than one billion legacy instructions. Concretely, the validation tests are mainly obtained from various real-world programs, including Linux, SPEC CINT2000, and Dhrystone (see Extended Data Fig.\ref{fig:spec_dhrystone}).  Extended Data Table~\ref{tab:hwchar} reports the area and power of the generated CPU core, where most of the circuits are implemented in combinational logic. On Dhrystone, the performance of generated CPU is comparable to that of the human-designed Intel 80486SX CPU (see Fig.\ref{fig:func_perf}b). 

\begin{figure}[t]
  \centering
  \includegraphics[width=\columnwidth]{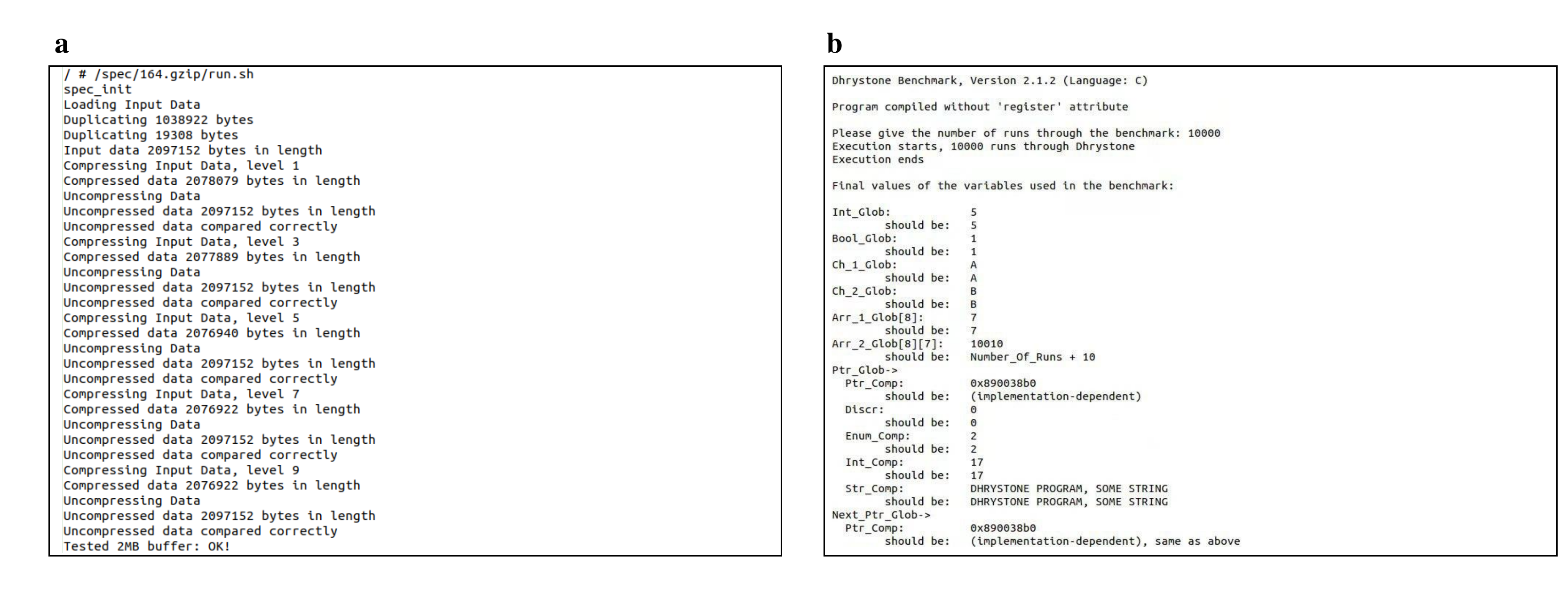}
  \caption{\textbf{More functional validation results.} \textbf{a}, The outputs of running one program (i.e., 164.gzip) of the SPEC CINT 2000 benchmark. \textbf{b}, The outputs of running the Dhrystone benchmark.}
  \label{fig:spec_dhrystone}
\end{figure}

Compared with traditional CPU design, the CPU designed by AI saves significant human efforts and time overhead. In traditional CPU design, human engineers do huge amounts of work for the functional and architecture design, Verilog HDL coding, and verification. For the K486 CPU (an Intel 486 compatible CPU), even the verification stage takes $190$ days\cite{yim1997design}, not to mention the time overhead of the entire design cycle. It takes increasingly more time overhead for modern CPUs with an even larger scale. In this work, we generate a BSD-based CPU core in about $5$ hours, which is about $1000\times$ reduction over our human-designed CPU core. 
To guarantee the apples-to-apples comparison of the design time, we only count the time consumption of working on the CPU core and do not include other parts in the SoC (System-on-Chip) or any back-end circuit design.

\subsection{Comparison with state-of-the-art.}

Emerging studies try to generate circuits based on machine-learning-based methods, and Nvidia announced a state-of-the-art study to generate an adder tree based on reinforcement learning (RL) methods\cite{roy2021prefixrl}.
However, the adder tree is significantly smaller than a CPU-scale circuit. The circuit logic of only one output bit of the generated CPU in one partition cluster contains $33,224$ gates, while that of the adder tree in Nvidia's work only contains $118$ gates (see Extended Data Fig.\ref{fig:BSD-RL}).
In fact, our approach is able to handle the circuit logic with $\sim33,000\times$ larger sizes than the RL-based method because the RISC-V CPU has more than $4$ million logic gates.

\begin{figure}[t]
  \centering
  \includegraphics[width=.9\columnwidth]{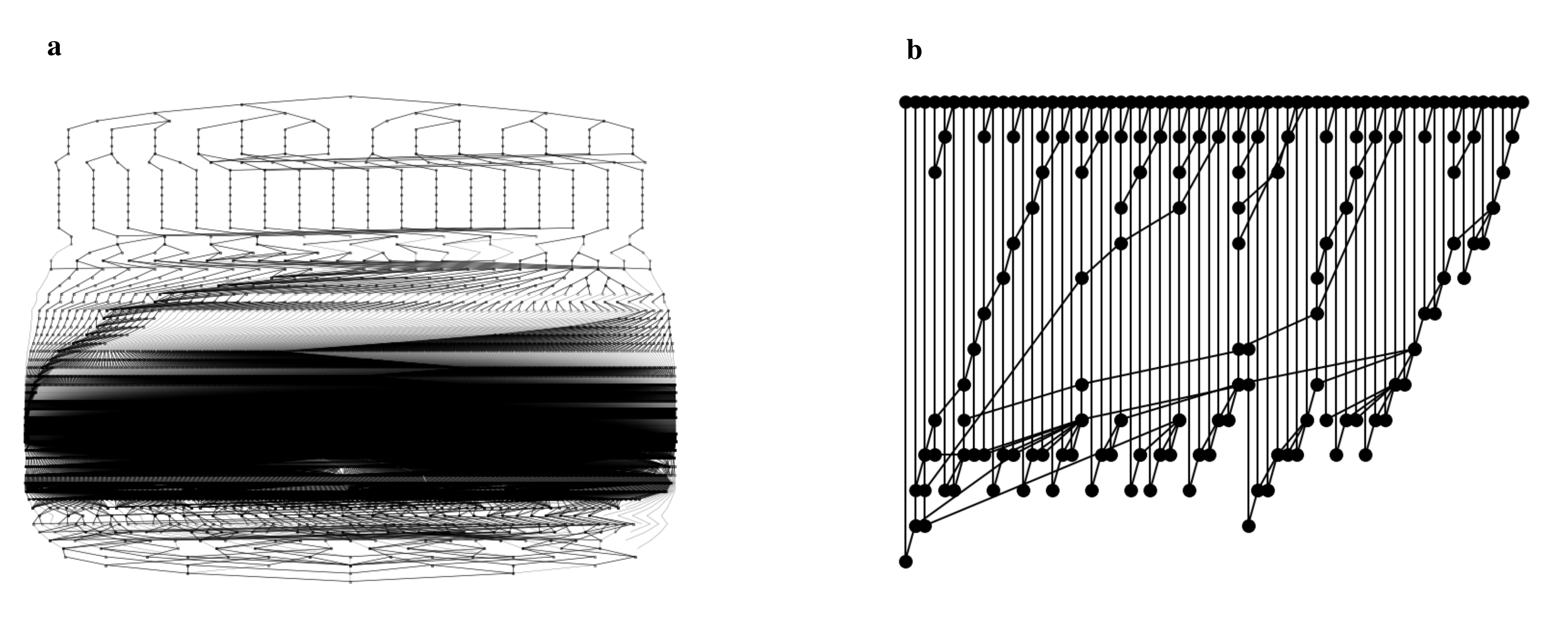}
  \caption{\textbf{Comparison of generated circuit logic of our approach and state-of-the-art RL-based approach.} \textbf{a}, A small part of the generated RISC-V CPU: one output bit in one partition cluster contains $33,224$ gates. \textbf{b}, The circuit logic of a 64-bit input adder only contains $118$ gates.}
  \label{fig:BSD-RL}
\end{figure}

\section*{Related work}

Machine design has long been a focus of the artificial intelligence and engineering community. Modern AI techniques have achieved substantial successes in the different design areas spanning from real-world substances like 
materials\cite{Liu2017MaterialsDA, Tabor2018AcceleratingTD, Saal2020MachineLI, Lee2022MethodologicalFF,raccuglia2016machine}, 
proteins\cite{russ2020evolution,Kuhlman2019AdvancesIP, Chen2019ToIP, Pan2021RecentAI}, 
and drugs\cite{Schneider2018AutomatingDD, ztrk2020ExploringCS, Vamathevan2019ApplicationsOM, Chen2018TheRO, Elton2019DeepLF, Stokes2020ADL, JimenezLuna2020DrugDW} 
to symbolic formulations like 
computer programs\cite{gulwani2017program, chaudhuri2021neurosymbolic, doi:10.1126/science.abq1158 }. 
However, previous work stays in the stage of designing relatively small objects.
Specifically, for materials, database OQMD\cite{Kirklin2015TheOQ} that discovered materials usually consists of tens of atoms, while the human-designed Ultra-high-molecular-weight polyethylene consists of millions of atoms\cite{kurtz2004ultrahigh}. For proteins, the AroQ family of chorismate mutases consists of $\approx100$ amino acids\cite{russ2020evolution}, while the human-designed protein, an artificial enzyme dubbed d-Dpo4-3C, consists of $352$ d-amino acids\cite{pech2017thermostable}. For drugs, one can estimate from the commonly-used database like GDB-17 and ZINC that the number of heavy atoms in a single molecule is usually $<100$\cite{polishchuk2013estimation, irwin2012zinc}, while the human-designed drug, i.e., Adalimumab, the first fully human monoclonal antibody approved by the US Food and Drug Administration (FDA),  contains $20,067$ atoms\cite{bang2004adalimumab}.
For computer programs, the machine-designed programs are either short programs within hundreds of characters in narrowly defined domain-specific languages\cite{Trivedi2021LearningTS}, or competitive-level programs with up to $2,304$ characters consisting of natural language descriptions and general-purpose programming languages\cite{doi:10.1126/science.abq1158}, which are still far smaller than real-world applications. Besides, the employed neural networks cannot guarantee the accuracy of the generated program. In contrast, the RISC-V CPU automatically designed by our approach requires exploring a search space of unprecedented size $10^{10^{540}}$, thus resulting in the combination of more than $4$ million 
logic gates and the CPU performs comparably against the human-designed commercial CPU as well.
Therefore, in terms of both the size of search spaces and the number of individual components, the CPU is the largest object, to our best knowledge, that can be designed by machines at present. Besides, these AI substances do not require high accuracy as the proposed process. These problems do not have a strict verification for their output design. For example, the average accuracy for protein designing is about 90$\%$, and the accuracy for computer program generation is about 66$\%$.

As one of the most important applications of machine design, automated circuit design has attracted attention from both academia and industry since the 1960s. Prior to the hardware programming languages, circuits are designed manually by manipulating the True Table or Karnaugh Map with hand-implemented circuit schemes. The hardware description languages (HDL) are then proposed, e.g., Verilog and VHDL in the 1980s, to help circuit designers by allowing them to focus on functional behaviours at register-transfer-level (RTL) abstraction.
To further raise the level of abstraction to boost CPU design, the HLS tools emerged to generate RTL description (e.g., HDL) from behavior specifications (e.g., C/C++ programs). There are three generations of HLS tools\cite{Martin09DT}. The first generation was developed as research prototypes in the 1980s\cite{paulin1989force}. It took domain-specific languages such as Silage as input and aimed at the design of a specific domain, i.e., signal processing\cite{de1986cathedral}. The second generation was available commercially in the mid-1990s\cite{elliott1999understanding}. Major EDA companies released their own products including Behavior Compiler\cite{knapp1996behavioral}, Visual Architect\cite{elliott1999understanding}, and Monet\cite{elliott1999understanding}. These tools still used the behavior HDL as input and thus were not widely adopted. Since the early 2000s, the third generation has been introduced by many vendors, and most of them used C, C++, SystemC, or Matlab as input\cite{nane2015survey,pursley2017high,HDLCoder}. This generation is capable of producing both dataflow and control logic with reasonable performance. Recently, machine learning, especially deep learning techniques, have been employed to improve the efficiency of HLS, e.g., HLS quality estimation\cite{dai_fast_2018,zhao_machine_2019,makrani_xppe_2019}, circuit performance predication\cite{ferianc_improving_2020,mohammadi_makrani_pyramid_2019,ustun_accurate_2020,yanghua_improving_2016}, overhead estimation\cite{mohammadi_makrani_pyramid_2019}, search space optimization\cite{kim_machine_2018,mahapatra_machine-learning_2014,wang_machine_2020,liu_efficient_2016,liu_learning-based_2013,meng_adaptive_2016}. 

With the breakthrough of machine learning and the increasing complexity of circuit design, machine learning techniques have been employed in different stages of the entire chip design flow\cite{huang2021machine}. Google used the AI method to design chip floorplans faster than humans based on reinforcement learning (RL)\cite{Mirhoseini21Nature}. Nvidia also presented an RL-based approach to design parallel prefix circuits such as adders or priority encoders\cite{roy2021prefixrl}. There is also research applying techniques such as deep convolutional neural networks (CNNs) and graph neural networks (GNNs) in the areas of automatic design space exploration, power analysis, VLSI physical design, and analog design\cite{khailany2020accelerating,venkatesan2019magnet,zhou2019primal,ren2020paragraph,lin2020abcdplace}.
As a widely used approximate algorithm to reduce complexity, the employed Monte-Carlo simulation has been used in other scenarios instead of circuit design. For example, in the game of Go, the Monte Carlo Tree Search is proposed in state-of-the-art AI methods such as AlphaGo\cite{Silver17Nature}. However, our proposed method completely differs from conventional Monte-Carlo-based game tree problems tackled in the current AI method. Though both of them utilize the basic concept of classic Monte-Carlo Simulation, AlphaGo uses Monte-Carlo Simulation to sample the game trees and learn the game strategy with neural networks, while our approach uses Monte-Carlo Simulation to sample the input-output examples and learn the circuit design with the proposed BSD.

All existing HDL/HLS tools and ML-based approaches require a formal description of the circuit logic, such as C, Verilog, and netlist, while our approach completely eliminates the manual effort to develop such formal inputs. Instead, the circuit design is automated by directly reusing empirical IO examples in the traditional verification process.

\label{sec:method}
\end{methods}












\end{document}